\documentclass[conference]{IEEEtran}
\IEEEoverridecommandlockouts

\newif\ifextended
\extendedtrue
\ifdefined\buildmode
  \def\buildmodecdc{cdc}
  \ifx\buildmode\buildmodecdc \extendedfalse \fi
\fi

\usepackage{natbib}
\usepackage{amsmath,amssymb,amsfonts,amsthm}
\usepackage{algorithm}
\usepackage{algorithmic}
\usepackage{graphicx}
\usepackage{textcomp}
\usepackage{xcolor}
\usepackage{url}
\def\BibTeX{{\rm B\kern-.05em{\sc i\kern-.025em b}\kern-.08em
    T\kern-.1667em\lower.7ex\hbox{E}\kern-.125emX}}

\newtheorem{theorem}{Theorem}

\DeclareMathOperator*{\argmax}{arg\,max}
\DeclareMathOperator*{\argmin}{arg\,min}
\DeclareMathOperator{\E}{\mathbb{E}}

\begin{document}

\title{Deep Adaptive Model-Based Design of Experiments\\
\thanks{\ifextended This work has been submitted to the IEEE for possible publication. Copyright may be transferred without notice, after which this version may no longer be accessible. This is the extended version; a six-page conference paper has been submitted to the 2026 IEEE Conference on Decision and Control (CDC).\else An extended version with proofs, training details, and ablation studies is available~\citep{strouwen2026extended}.\fi{} Code: \protect\url{https://github.com/arno-papers/CDC2026}.}
}

\author{\IEEEauthorblockN{Arno Strouwen}
\IEEEauthorblockA{\textit{Strouwen Statistics}, Hasselt, Belgium \\
\textit{Biosystems Department, KU Leuven}, Belgium \\
contact@arnostrouwen.com}
\and
\IEEEauthorblockN{Sebastian Micluța-Câmpeanu}
\IEEEauthorblockA{\textit{JuliaHub}, Massachusetts, USA \\
\textit{Faculty of Physics, University of Bucharest}, Romania \\
sebastian.mc95@proton.me}
}

\maketitle

\begin{abstract}
Model-based design of experiments (MBDOE) is essential for efficient parameter estimation in nonlinear dynamical systems. However, conventional adaptive MBDOE requires costly posterior inference and design optimization between each experimental step, precluding real-time applications. We address this by combining Deep Adaptive Design (DAD), which amortizes sequential design into a neural network policy trained offline, with differentiable mechanistic models. For dynamical systems with known governing equations but uncertain parameters, we extend sequential contrastive training objectives to handle nuisance parameters and propose a transformer-based policy architecture that respects the temporal structure of dynamical systems. We demonstrate the approach on four systems of increasing complexity: a fed-batch bioreactor with Monod kinetics, a Haldane bioreactor with uncertain substrate inhibition, a two-compartment pharmacokinetic model with nuisance clearance parameters, and a DC motor for real-time deployment. On the Monod bioreactor, the adaptive policy reduces posterior RMSE by about 30\% versus the best static design; on the DC motor, it evaluates in $24~\mu$s per step versus $7.4~\mathrm{ms}$ for non-amortized adaptive design, enabling real-time deployment.
\end{abstract}


\section{Introduction}
Model-based design of experiments (MBDOE) is a cornerstone methodology for efficient parameter estimation in nonlinear dynamical systems~\citep{franceschini,walter1997}. By leveraging mechanistic knowledge of the system, MBDOE selects experimental conditions that maximize information about unknown parameters, enabling precise estimation with fewer experiments than naive approaches.

Bayesian optimal experimental design (BOED) provides a principled framework for this problem~\citep{lindley1956}, selecting designs that maximize the expected information gain (EIG). In sequential settings, adaptive strategies that use past observations to guide future decisions are particularly powerful, but conventional adaptive BOED requires costly posterior inference and design optimization between each experimental step. The EIG objective is doubly intractable~\citep{rainforth2018nesting}, creating a computational bottleneck that precludes real-time applications.

Deep Adaptive Design (DAD)~\citep{foster2021} addresses this bottleneck by amortizing sequential design into a neural network policy trained offline. iDAD~\citep{ivanova2021implicit} extended this to implicit (likelihood-free) models.
However, for the mechanistic models central to MBDOE, explicit likelihoods \emph{are} available via differentiable simulation~\citep{rackauckas}. This allows us to use the sequential Prior Contrastive Estimation (sPCE) bounds of the original DAD framework, which require explicit likelihoods, and to differentiate end-to-end through the simulator for low-variance gradient estimates.

In this paper, we combine DAD with differentiable mechanistic models to enable real-time adaptive MBDOE for dynamical systems with known governing equations and uncertain parameters.

Our contributions are:
\begin{itemize}
\item We formulate adaptive MBDOE for dynamical systems where the model structure is known, training the policy end-to-end by differentiating through the ODE solver and observation model.
\item We extend the sPCE bound to handle nuisance parameters, enabling targeted inference by marginalizing over nuisance parameters such as measurement noise.
\item We retain the transformer-based policy architecture from DAD and adapt it to dynamical systems by making temporal order explicit through positional encoding, rather than treating design-outcome pairs as permutation invariant.
\item We demonstrate the methodology on four examples: a Monod bioreactor, a Haldane bioreactor illustrating adaptation to target parameters, a pharmacokinetic model illustrating nuisance adaptation together with a brief architecture ablation, and a DC motor for real-time deployment; on the Monod bioreactor the adaptive policy reduces posterior RMSE by about 30\% relative to the best static design, and on the DC motor it evaluates in $24~\mu$s per step versus $7.4~\mathrm{ms}$ for non-amortized adaptive design.
\end{itemize}

\section{Model}
\label{sec:model}
We consider dynamical systems of the form:
\begin{equation}\label{eq:system}
\begin{aligned}
\frac{dx}{dt} &= f(t, x, \theta_T, \theta_N, u(t)), \quad x(0) = x_0(\theta_N),\\
y_k \mid x(t_k), \theta_N &\sim \mathcal{N}(g(x(t_k)),\, \Sigma(x(t_k), \theta_N)), \quad k = 1, \ldots, K,
\end{aligned}
\end{equation}
where $x(t) \in \mathbb{R}^n$ is the state vector, $u(t) \in \mathbb{R}^m$ is a controllable input vector, and $y_k \in \mathbb{R}^q$ is the observation vector at measurement time $t_k$.
We partition parameters as $\theta = (\theta_T, \theta_N)$, where $\theta_T$ are the target parameters of interest and $\theta_N$ are nuisance parameters (e.g., measurement noise scale, uncertain initial conditions).
Some components of the initial state $x_0$ may depend on $\theta_N$.
The observation noise covariance $\Sigma(x(t_k), \theta_N) \in \mathbb{R}^{q \times q}$ may depend on the state (e.g., proportional noise).
We assume a fixed experimental schedule with known measurement times $0 < t_1 < \cdots < t_K$ and piecewise-constant inputs, $u(t)=u_k$ for $t \in [t_{k-1}, t_k)$.
The design goal is to choose the input profile $u(t)$ (equivalently, the sequence $u_{1:K}$) to maximize the information gained about $\theta_T$.

We instantiate this framework on a well-mixed fed-batch bioreactor~\citep{versyck}, with state $x = (C_s, C_x, V)$, input $u = Q_{in}$, and kinetic parameters $\theta_T = (\mu_{\max}, K_s)$.
We consider a $14$-hour experiment with measurements collected every hour, so $K=14$ and $t_k = k\,\mathrm{h}$ for $k=1,\ldots,K$ (with $t_0=0\,\mathrm{h}$).
We choose a feed rate $Q_{in,k}$ that is held constant on $[(k-1)\,\mathrm{h}, k\,\mathrm{h})$.
The dynamics are:
\begin{equation}\label{eq:bioreactor}
\begin{aligned}
\frac{dC_s}{dt} &= -\frac{\mu(C_s; \theta_T)}{Y_{xs}} C_x + \frac{Q_{in}}{V}(C_{s,\text{in}} - C_s),\\
\frac{dC_x}{dt} &= \mu(C_s; \theta_T) C_x - \frac{Q_{in}}{V}C_x,\\
\frac{dV}{dt} &= Q_{in},
\end{aligned}
\end{equation}
where $C_s$ is substrate concentration, $C_x$ is biomass concentration, $V$ is reactor volume, $Q_{in}$ is the controllable feed rate, and $Y_{xs} = 0.777$ and $C_{s,\text{in}} = 50$ are the known yield coefficient and feed substrate concentration, respectively.

The specific growth rate $\mu$ follows Monod kinetics:
\begin{equation}
\mu(C_s; \theta_T) = \frac{\mu_{\max} C_s}{K_s + C_s},
\end{equation}
where $\mu_{\max}$ and $K_s$ are unknown, with independent uniform priors $\mu_{\max} \in [0.3, 0.5]$ and $K_s \in [0.3, 0.6]$.
We measure substrate concentration (so $q=1$ and $g(x)=C_s$), with observations $y_k \sim \mathcal{N}(C_s(t_k),\, \sigma^2)$.
The initial substrate concentration and volume are known: $C_{s,0} = 3.0$ and $V_0 = 7.0$.
The nuisance parameters are $\theta_N = (\sigma, C_{x,0})$: the initial biomass $C_{x,0} \sim \text{Uniform}(0.10, 0.50)$ varies between batches due to different inocula, and the noise scale $\sigma \sim \text{Uniform}(0.05, 0.15)$ is unknown but constant across measurements. Both are treated as unknown and marginalized during training.
The goal is to design a sequence of feed rates $Q_{in,1}, \ldots, Q_{in,K}$ that maximizes information about $\mu_{\max}$ and $K_s$ from noisy $C_s$ measurements.

\section{Information-Theoretic Criterion}
\label{sec:criterion}

In Bayesian experimental design, we begin with a prior $p(\theta)$ encoding our uncertainty about parameters $\theta = (\theta_T, \theta_N)$ and seek an input $u$ such that the posterior on the target parameters $p(\theta_T|y,u)$ is as concentrated as possible. The standard measure of concentration is entropy:
\begin{equation}
H(\theta_T|y,u) = -\int p(\theta_T|y,u) \log p(\theta_T|y,u)\,d\theta_T.
\end{equation}
Since the outcome $y$ is unknown at the design stage, we average over all possible observations weighted by their marginal likelihood $p(y|u) = \int p(y|\theta,u)\,p(\theta)\,d\theta$:
\begin{equation}\label{eq:bed_entropy}
u^\star = \argmin_{u} \int H(\theta_T|y,u)\,p(y|u)\,dy.
\end{equation}
This is equivalent to maximizing the targeted expected information gain (EIG), the mutual information between target parameters and observations~\citep{lindley1956}:
\begin{equation}\label{eq:eig}
\mathcal{I}(u) = I(\theta_T; y|u) = H(\theta_T) - \E_{p(y|u)}[H(\theta_T|y,u)].
\end{equation}
The EIG is intractable to compute in general because $p(y|u)$ requires integrating over the full parameter space. Similarly, obtaining the posterior $p(\theta_T|y,u)$ requires the partial marginal likelihood $p(y|\theta_T,u) = \int p(y|\theta_T, \theta_N, u)\,p(\theta_N|\theta_T)\,d\theta_N$, which integrates over the nuisance parameters.

The formulation above is a \emph{static} design: all inputs $u_{1:K}$ are chosen upfront by maximizing the EIG over the full trajectory.
In sequential experiments, however, the design should adapt to past observations.
We therefore optimize over \emph{policies}---decision rules $\pi_k(h_{k-1})$ that map the experimental history $h_{k-1} = ((u_1, y_1), \ldots, (u_{k-1}, y_{k-1}))$ to the next input $u_k$.
This policy formulation is equivalent to the usual nested adaptive criterion.\ifextended{} See Appendix~\ref{sec:policy_equiv} for a formal proof. \else{} See the extended version~\citep{strouwen2026extended} for a formal proof. \fi
Once an optimal policy is available, online deployment requires only evaluating the decision rule at each step, avoiding repeated optimization during the experiment. For continuous observation spaces, tabular dynamic programming is infeasible; instead, the policy is parameterized by a neural network and trained offline via Monte Carlo simulation, as we do here.

Let $\theta = (\theta_T, \theta_N)$ denote all unknown parameters, $h_k = ((u_1, y_1), \ldots, (u_k, y_k))$ the experimental history, and $u_k = \pi(h_{k-1})$ the policy.
The targeted EIG from the full sequential experiment is:
\begin{equation}
\mathcal{I}_K^{\text{tgt}}(\pi) = \E_{p(\theta)p(h_K|\theta,\pi)}\left[\log \frac{p(h_K|\theta_T,\pi)}{p(h_K|\pi)}\right] = I(\theta_T; h_K | \pi),
\end{equation}
where $p(h_K|\theta_T, \pi) = \int p(h_K|\theta, \pi)\, p(\theta_N|\theta_T) \, d\theta_N$ and $p(h_K|\pi) = \int p(h_K|\theta,\pi)\,p(\theta)\,d\theta$.

Direct optimization of $\mathcal{I}_K^{\text{tgt}}(\pi)$ is intractable due to these marginals. \citet{foster2021} introduced sequential Prior Contrastive Estimation (sPCE), which uses contrastive samples to obtain a tractable lower bound on the sequential EIG. We extend their framework to handle nuisance parameters, yielding a bound on the targeted EIG. Intuitively, the objective rewards policies whose induced histories make the true target parameters easier to distinguish from contrastive alternatives after marginalizing over nuisance uncertainty.

\begin{theorem}[Targeted sPCE]
\label{thm:targeted_spce_main}
Let $\theta^{(0)} = (\theta_T^{(0)}, \theta_N^{(0)}) \sim p(\theta)$ generate $h_K$, let $\theta^{(1)},\ldots,\theta^{(L)} \sim p(\theta)$ be contrastive samples, and let $\tilde{\theta}_N^{(1:M)} \sim p(\theta_N|\theta_T^{(0)})$ be nuisance samples. The objective
\begin{equation}
\hat{\mathcal{L}}_K^{\text{tgt}}(\pi, L, M) = \E\left[\log \frac{\frac{1}{M}\sum_{m=1}^{M} p(h_K|\theta_T^{(0)}, \tilde{\theta}_N^{(m)}, \pi)}{\frac{1}{L+1}\sum_{\ell=0}^{L} p(h_K|\theta^{(\ell)}, \pi)}\right]
\end{equation}
satisfies $\hat{\mathcal{L}}_K^{\text{tgt}} \leq \mathcal{I}_K^{\text{tgt}}$, converging to $\mathcal{I}_K^{\text{tgt}}$ as $L, M \to \infty$.
\ifextended See Appendix~\ref{sec:nuisance} for proof.\else{} See the extended version~\citep{strouwen2026extended} for proof.\fi
\end{theorem}

\section{Optimal Policy}
We parameterize the sequential design policy as a neural network $u_k = \pi_\phi(h_{k-1})$ that maps the experimental history to the next input.
Here $\phi$ are the trainable policy parameters.
The policy is trained offline; at deployment, computing $u_k$ requires only a forward pass.

\subsection{Training}
Training maximizes the targeted contrastive objective $\hat{\mathcal{L}}_K^{\text{tgt}}$ by stochastic gradient ascent with Adam on $\phi$. Since observations depend on $\phi$ through the policy, the expectation is over $\phi$-dependent random variables. For Gaussian observation models, we apply the reparameterization trick: writing $y_k = \mu_k + \Sigma_k^{1/2}\varepsilon_k$ with $\varepsilon_k \sim \mathcal{N}(0, I)$ independent of $\phi$ allows the gradient to pass inside the expectation:
\begin{equation}
\begin{aligned}
\frac{d\hat{\mathcal{L}}_K^{\text{tgt}}}{d\phi}
&= \E\Bigg[\frac{d}{d\phi}\Bigg(
\log\Big(\frac{1}{M}\sum_{m=1}^{M}
p(h_K|\theta_T^{(0)}, \tilde{\theta}_N^{(m)}, \pi_\phi)\Big)\\
&\qquad - \log\Big(\frac{1}{L+1}\sum_{\ell=0}^{L}
p(h_K|\theta^{(\ell)}, \pi_\phi)\Big)
\Bigg)\Bigg],
\end{aligned}
\end{equation}
where the expectation is approximated with a minibatch of $B$ independent episodes (Algorithm~\ref{algo:DAD}).
Each episode draws a ground-truth sample $\theta^{(0)} \sim p(\theta)$ that generates the trajectory, $L$ contrastive samples $\theta^{(1:L)} \sim p(\theta)$ for the denominator, and $M$ nuisance samples $\tilde{\theta}_N^{(1:M)} \sim p(\theta_N|\theta_T^{(0)})$ for the numerator, all sharing the target parameters $\theta_T^{(0)}$ of the ground truth. This requires $L + M + 2$ ODE solves per episode: one rollout to generate $h_K$, $M$ nuisance-conditioned likelihood evaluations for the numerator, and $L + 1$ likelihood evaluations for the denominator.

\begin{algorithm}[t]
\caption{Targeted Deep Adaptive Design for MBDOE}
\label{algo:DAD}
\begin{algorithmic}[1]
\REQUIRE Prior $p(\theta_T, \theta_N)$, dynamics $f$, output $g$, steps $K$, minibatch size $B$
\ENSURE Trained policy network $\pi_\phi$
\WHILE{training budget not exceeded}
    \STATE \textit{(All likelihood evaluations below are conditioned on the current policy $\pi_\phi$.)}
    \STATE Sample $\{\theta^{(\ell,b)}\}_{\ell=0:L,\,b=1:B} \overset{\text{i.i.d.}}{\sim} p(\theta)$
    \STATE Sample $\{\tilde{\theta}_N^{(m,b)}\}_{m=1:M,\,b=1:B} \sim p(\theta_N|\theta_T^{(0,b)})$
    \FOR{$b = 1, \ldots, B$}
        \STATE Set $x_0^{(b)}$ from $\theta_N^{(0,b)}$, history $h_0^{(b)} = \varnothing$
        \FOR{$k = 1, \ldots, K$}
            \STATE Compute input $u_k^{(b)} = \pi_\phi(h_{k-1}^{(b)})$
            \STATE Simulate $x_k^{(b)}$ from \eqref{eq:system} with $\theta^{(0,b)}$ and input $u_k^{(b)}$
            \STATE Sample $y_k^{(b)} \sim \mathcal{N}(g(x_k^{(b)}),\, \Sigma(x_k^{(b)}, \theta_N^{(0,b)}))$
            \STATE Update history $h_k^{(b)} = (h_{k-1}^{(b)}, (u_k^{(b)}, y_k^{(b)}))$
        \ENDFOR
        \STATE Compute $\widehat p_M^{(b)} = \frac{1}{M}\sum_{m=1}^{M} p(h_K^{(b)}|\theta_T^{(0,b)}, \tilde{\theta}_N^{(m,b)})$
        \STATE Compute $\widehat p_L^{(b)} = \frac{1}{L+1}\sum_{\ell=0}^{L} p(h_K^{(b)}|\theta^{(\ell,b)})$
        \STATE Set $r^{(b)} = \log \widehat p_M^{(b)} - \log \widehat p_L^{(b)}$
    \ENDFOR
    \STATE Compute gradient $\nabla_\phi \frac{1}{B}\sum_{b=1}^{B} r^{(b)}$ via autodiff
    \STATE Update $\phi$ via Adam
\ENDWHILE
\end{algorithmic}
\end{algorithm}

\subsection{Choosing $L$, $M$, and $B$}
The targeted sPCE objective depends on the number of contrastive samples $L$ and nuisance samples $M$, while the optimizer uses a minibatch size $B$ to form stochastic gradient estimates.
These hyperparameters trade off bound tightness, estimator variance, and computational cost.
\ifextended See Appendix~\ref{sec:budget_stub} for a brief discussion and asymptotic guidance under a fixed compute budget.\else{} See the extended version~\citep{strouwen2026extended} for a brief discussion and asymptotic guidance under a fixed compute budget.\fi

\subsection{Network Architecture}
Following DAD~\citep{foster2021}, we use a transformer-based policy architecture. For dynamical systems, however, the key additional requirement is to represent temporal order explicitly, since earlier inputs affect later states and therefore later observations.

We therefore augment the history with sinusoidal positional encodings and use causal attention to map $((u_1, y_1), \ldots, (u_{k-1}, y_{k-1}))$ to the next input:
\begin{equation}
u_k = \pi_\phi(h_{k-1}) = F_{\phi_2}(\text{Transformer}_{\phi_1}(h_{k-1})),
\end{equation}
where $\text{Transformer}_{\phi_1}$ processes the sequence and $F_{\phi_2}$ is a final projection head.

\ifextended All examples share the same architecture and optimizer; full hyperparameters are given in Appendix~\ref{sec:training_details}.\else All examples share the same architecture and optimizer; full hyperparameters are given in the extended version~\citep{strouwen2026extended}.\fi

\section{Results}

\begin{figure*}[t]
\centering
\includegraphics[width=\textwidth]{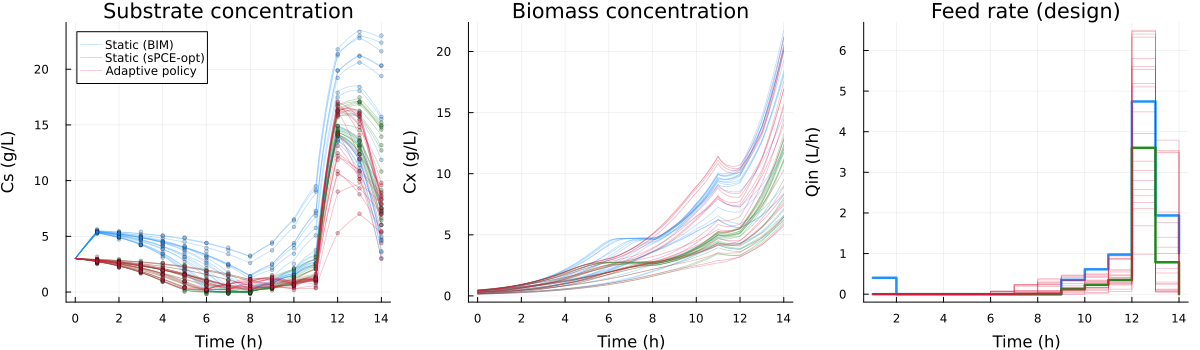}
\caption{Design comparison for the Monod bioreactor (20 rollouts per strategy, shared prior samples). Left to right: substrate $C_s$, biomass $C_x$, and feed rate $Q_{in}$ over 14~hours. The adaptive policy varies its design per rollout, keeping the feed near zero during initial growth before ramping up. All three strategies concentrate feeding towards the end; the adaptive policy adapts the timing and magnitude to each realization.}
\label{fig:dynamics}
\end{figure*}

\begin{table}[t]
\centering
\caption{Monod bioreactor: targeted sPCE scores (mean $\pm$ SEM, $n=1000$; higher is better) and posterior RMSE ($\times 10^3$, $n=5000$; lower is better).}
\label{tab:spce}
\begin{tabular}{@{}lccc@{}}
\hline
Design strategy & sPCE score & \multicolumn{2}{c}{RMSE ($\times 10^3$)} \\
 & & $\mu_{\max}$ & $K_s$ \\
\hline
Adaptive policy & $\mathbf{3.68 \pm 0.04}$ & $\mathbf{3.1}$ & $\mathbf{24.9}$ \\
Static (sPCE-opt) & $3.34 \pm 0.04$ & $5.6$ & $41.5$ \\
Static (BIM) & $3.45 \pm 0.04$ & $4.4$ & $34.3$ \\
\hline

\end{tabular}
\end{table}

\begin{figure}[t]
\centering
\includegraphics[width=\columnwidth]{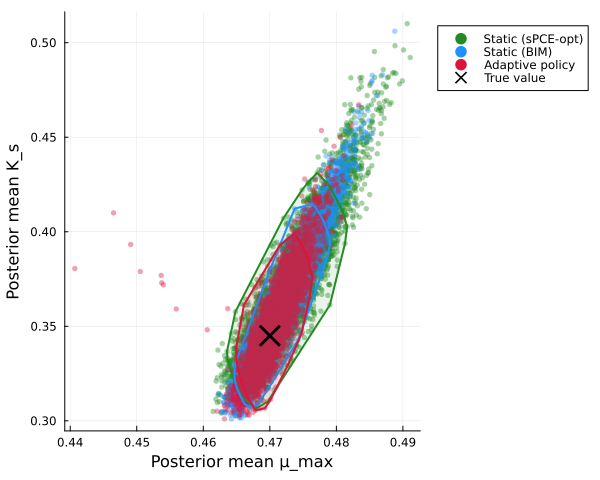}
\caption{Posterior mean estimates ($\hat{\mu}_{\max}$, $\hat{K}_s$) across 5000 trials for three design strategies. Convex hulls (90\% closest to median) indicate the spread of posterior means; the adaptive policy yields the most concentrated estimates around the true values.}
\label{fig:posterior}
\end{figure}

\begin{figure*}[t]
\centering
\includegraphics[width=\textwidth]{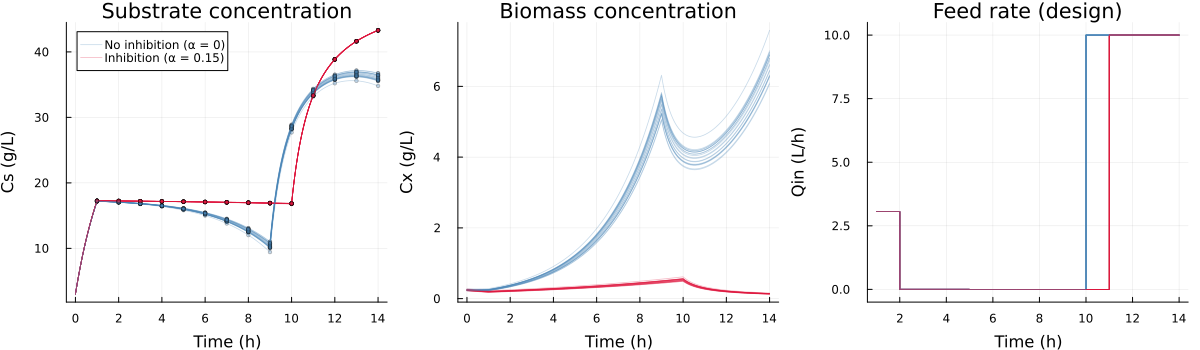}
\caption{Haldane bioreactor rollouts (20 per scenario, shared nuisance samples). Blue: no inhibition ($\alpha = 0$); red: strong inhibition ($\alpha = 0.15$). The policy first applies the same probing feed pulse in both cases, pushing the reactor into a regime where inhibition can become visible. It then adapts to the true target value: without inhibition the policy resumes feeding earlier to avoid returning to a low-substrate, non-inhibitory regime, whereas under strong inhibition the final ramp to maximal feed is delayed.}
\label{fig:haldane}
\end{figure*}

\paragraph{Monod bioreactor.}
Fig.~\ref{fig:dynamics} shows representative rollouts: the adaptive policy varies its feed profile per rollout in response to observations, while the static designs apply a fixed sequence regardless of the outcome.
Table~\ref{tab:spce} compares the adaptive policy against two static baselines---a classical Bayesian D-optimal design (BIM\ifextended{}, detailed in Appendix~\ref{sec:bim}\else{}; see the extended version~\citep{strouwen2026extended}\fi) and a static design optimized with the same sPCE objective. The adaptive policy significantly outperforms both on sPCE score ($p < 0.001$, paired $t$-tests) and reduces posterior RMSE by approximately 30\% relative to the best static design.
Fig.~\ref{fig:posterior} confirms that the adaptive policy yields the most concentrated posterior estimates around the true parameter values.
Only 10/5000 adaptive trials (the outliers in Fig.~\ref{fig:posterior}) yield poor posteriors. In these cases, rare misleading early observations cause the policy to stop pumping, driving substrate to zero. These are training failures: with a larger training budget, the policy learns to avoid them. All reported statistics include these trials.

\paragraph{Haldane substrate inhibition.}
This example shares the Monod bioreactor setup (same states, feed rate design $Q_{in}$, substrate observation $C_s$) but replaces the growth rate with Haldane kinetics: $\mu = \mu_{\max} C_s / (K_s + C_s + C_s^2/K_i)$, adding substrate inhibition. We set $\alpha = 1/K_i$ as the sole target with a uniform prior $\alpha \sim \text{Uniform}(0, 0.15)$, and treat $\mu_{\max}$, $K_s$, $\sigma$, $C_{x,0}$ as tightly bounded nuisance parameters (e.g., $\mu_{\max} \in [0.39, 0.41]$, $K_s \in [0.44, 0.46]$), so the sPCE objective focuses on discriminating whether inhibition is present ($\alpha = 0$ recovers Monod).
Unlike the Monod example, where the policy adapts to nuisance variation, here it must adapt to the \emph{target parameter} itself.
Fig.~\ref{fig:haldane} shows this qualitatively. The common initial feed pulse is itself informative: it pushes the reactor into the substrate range where inhibition could become active, so the ensuing response becomes highly diagnostic of the true value of $\alpha$. The policy then adapts to that inferred target value. Without inhibition, biomass grows rapidly and would otherwise drive the system back into a low-substrate, non-inhibitory regime that is less informative about $\alpha$, so the policy resumes feeding earlier. Under strong inhibition, growth remains suppressed and substrate stays elevated for longer, so the final ramp to maximal feed is delayed. In this way, the timing of the later feed switch becomes a direct consequence of the inferred inhibition strength.

\paragraph{Pharmacokinetic model.}
A two-compartment PK model with transit absorption describes drug distribution between a central compartment (volume $V_C$) and a peripheral compartment (volume $V_P$), fed by a three-stage transit chain:
\begin{equation}\label{eq:pk}
\begin{aligned}
\dot{A}_{t_1} &= R_\text{inf} - k_{tr} A_{t_1}, \\
\dot{A}_{t_2} &= k_{tr} A_{t_1} - k_{tr} A_{t_2}, \\
\dot{A}_{t_3} &= k_{tr} A_{t_2} - k_a A_{t_3}, \\
\dot{A}_c &= k_a A_{t_3} - \tfrac{CL + Q_d}{V_C} A_c + \tfrac{Q_d}{V_P} A_p, \\
\dot{A}_p &= \tfrac{Q_d}{V_C} A_c - \tfrac{Q_d}{V_P} A_p,
\end{aligned}
\end{equation}
where $R_\text{inf}$ is the controllable infusion rate (design input). The target parameters are the absorption rates $\theta_T = (k_a, k_{tr})$, while $\theta_N = (CL, Q_d, \sigma_{\text{prop}}, \sigma_{\text{add}})$ are nuisance, including patient-specific clearance. The observation is central concentration $C_c = A_c/V_C$ with combined proportional and additive noise: $y_k \sim \mathcal{N}(C_c, (\sigma_{\text{prop}} C_c)^2 + \sigma_{\text{add}}^2)$.
We simulate a $24$-hour experiment with hourly observations ($K = 24$), using uniform priors $k_a, k_{tr} \in [0.5, 3.0]$, $CL \in [1.0, 5.0]$, $Q_d \in [0.5, 3.0]$, $\sigma_{\text{prop}} \in [0.05, 0.20]$, and $\sigma_{\text{add}} \in [0.01, 0.10]$.
Fig.~\ref{fig:pharma} provides a qualitative illustration that nuisance parameters can materially change the optimal dosing schedule. After a common initial infusion pulse, the policy resumes dosing earlier under fast elimination and later under slow elimination, reflecting how quickly the drug is cleared. This nuisance adaptation emerges from the targeted sPCE objective, which marginalizes over $(CL, Q_d)$ during training, so the policy learns that the dosing schedule most informative for probing absorption depends on the elimination profile.
A brief architecture ablation on this example suggests that explicit temporal encoding is the main architectural ingredient: removing positional encoding drops sPCE by $0.35$~nats ($t{=}12.1$), while replacing attention with a comparably-sized MLP causes a smaller drop of $0.03$~nats ($t{=}2.2$)\ifextended{} (Appendix~\ref{sec:ablation})\else{} (see the extended version~\citep{strouwen2026extended})\fi.

\begin{figure*}[t]
\centering
\includegraphics[width=\textwidth]{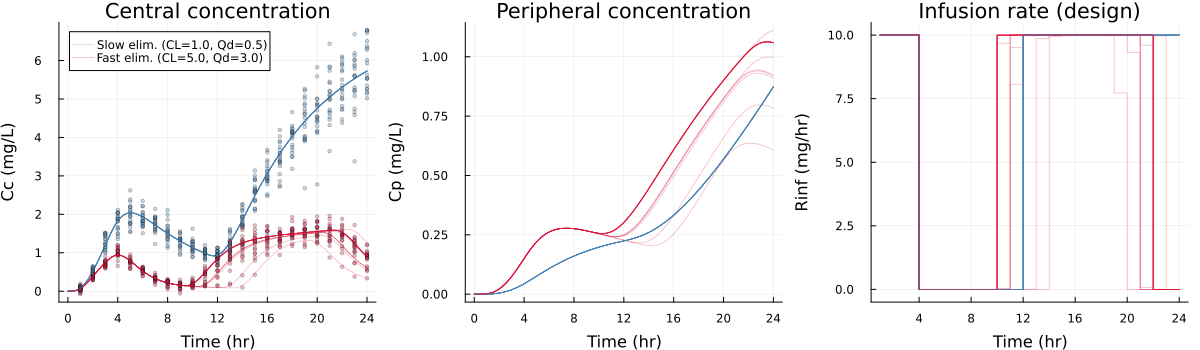}
\caption{PK model rollouts (20 per scenario, shared absorption parameters $k_a$, $k_{tr}$ at midrange). Blue: slow elimination ($CL = 1.0$~L/hr, $Q_d = 0.5$~L/hr); red: fast elimination ($CL = 5.0$~L/hr, $Q_d = 3.0$~L/hr). Left to right: central concentration $C_c = A_c/V_C$, peripheral concentration $C_p = A_p/V_P$, and infusion rate $R_\text{inf}$ (design). After a common initial infusion pulse, the policy resumes dosing earlier under fast elimination, whereas under slow elimination it delays redosing because the drug persists longer.}
\label{fig:pharma}
\end{figure*}

\paragraph{DC motor.}
A DC motor with state $(\omega, i)$ (angular velocity, current) follows
\begin{equation}\label{eq:motor}
L_a \frac{di}{dt} = V_{in} - Ri - k\omega, \quad J \frac{d\omega}{dt} = ki - f\omega,
\end{equation}
with known resistance $R = 0.5\,\Omega$ and inductance $L_a = 4.5\,$mH.
The design input is voltage $V_{in} \in [0, 10]\,$V and the observation is angular velocity $\omega$ with additive Gaussian noise.
The experiment lasts $100\,$ms ($K = 10$ steps at $\Delta t = 10\,$ms), with target parameters $\theta_T = (k, J)$ (motor constant, inertia), nuisance $\theta_N = (f, \sigma)$ (friction, noise), and uniform priors $k \in [0.3, 0.7]$, $J \in [0.01, 0.04]$, $f \in [0.005, 0.02]$, $\sigma \in [0.5, 2.0]$.
We compare the amortized adaptive policy against an adaptive Bayesian D-optimal baseline that re-optimizes the design at each step using the current posterior mode (adaptive BIM)\ifextended{} (Appendix~\ref{sec:bim})\else{} (see the extended version~\citep{strouwen2026extended})\fi.
Table~\ref{tab:motor} shows that the amortized policy outperforms adaptive BIM on both sPCE score and posterior RMSE. This is expected because the learned policy is trained against the full sequential objective, whereas adaptive BIM re-optimizes a myopic surrogate based on the current posterior mode. The amortized policy also evaluates orders of magnitude faster: a single forward pass takes a median of $24\,\mu$s versus $7.4\,$ms for online re-optimization.
Fig.~\ref{fig:timing} confirms that even the 99.9th percentile ($79\,\mu$s) remains $126\times$ below the $10\,$ms sampling interval, making real-time deployment feasible.
Note that this adaptive BIM baseline is already a deliberately simple strategy---myopic, with a Dirac posterior approximation and a 1D grid search---yet it already fails to run within the sampling interval by the later steps, as each FIM evaluation must differentiate through the growing history of design points.

\begin{table}[t]
\centering
\caption{DC motor: adaptive sPCE policy vs.\ adaptive BIM (online re-optimization). sPCE scores (mean $\pm$ SEM, $n=500$; higher is better), posterior RMSE ($\times 10^3$, $n=500$; lower is better), and per-step wall time.}
\label{tab:motor}
\small
\begin{tabular}{@{}lcccc@{}}
\hline
 & sPCE & \multicolumn{2}{c}{RMSE ($\times 10^3$)} & Time/step \\
 & & $k$ & $J$ & \\
\hline
Adaptive (sPCE) & $\mathbf{2.96 \pm 0.06}$ & $\mathbf{20.6}$ & $\mathbf{2.9}$ & $24\,\mu$s \\
Adaptive (BIM) & $2.56 \pm 0.06$ & $35.1$ & $4.6$ & $7.4\,$ms \\
\hline

\end{tabular}
\end{table}

\begin{figure}[t]
\centering
\includegraphics[width=\columnwidth]{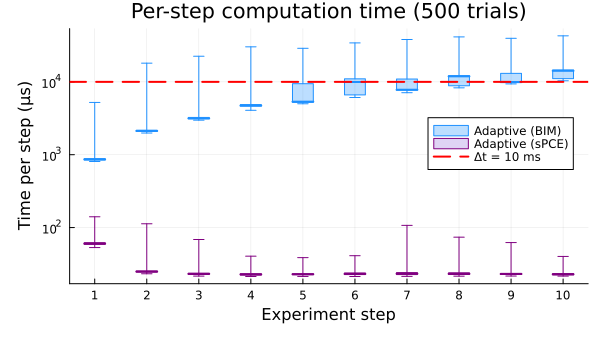}
\caption{Per-step computation time (500~rollouts, single-threaded, AMD Ryzen~9 5900X). Boxes show IQR; whiskers the 0.1--99.9th percentile. Adaptive BIM (blue) grows from ${\sim}1\,$ms to ${\sim}13\,$ms, exceeding the $\Delta t = 10\,$ms sampling interval (dashed red). The amortized sPCE policy (purple) evaluates in ${\sim}24\,\mu$s regardless of step.}
\label{fig:timing}
\end{figure}

\section{Conclusion}

We have shown that combining deep adaptive design with differentiable mechanistic models enables real-time adaptive model-based design of experiments.
The amortized policy outperforms both static baselines and adaptive online re-optimization across four systems of increasing complexity.

Current limitations remain: the policy is trained for a specific prior distribution, significant prior shifts require retraining, model misspecification is not addressed, stochasticity is limited to the measurement equation, and training requires GPU hours on a supercomputer, although deployment is lightweight.
Promising next steps are to combine the framework with Bayesian model selection~\citep{strouwen2026modelselection} for model uncertainty, extend the framework to process noise through Bayesian filtering~\citep{strouwen2023adaptive}, and deploy the learned policy on physical hardware with closed-loop feedback.
Together, these extensions would move the framework closer to robust real-world adaptive experimentation.

\section*{Acknowledgment}
The resources and services used in this work were provided by the VSC (Flemish Supercomputer Center), funded by the Research Foundation -- Flanders (FWO) and the Flemish Government.

\noindent\textit{Use of generative AI:}
Claude~\citep{anthropic2025claude} was used as a programming assistant for implementing the Julia codebase and as a writing aid for editing and improving manuscript text. All methodological decisions, experimental design, result interpretation, and scientific claims are the authors' own. The authors verified all outputs and take full responsibility for the content.

\bibliographystyle{plainnat}
\bibliography{references}

\ifextended
\newpage
\onecolumn
\appendices

\section{Targeted Inference with Nuisance Parameters}
\label{sec:nuisance}

We extend the sequential contrastive estimation framework to handle nuisance parameters—parameters that must be modeled but are not the focus of inference.
Since $p(h_K|\theta_T,\pi)$ is intractable, we estimate the two required marginals by Monte Carlo:
\begin{itemize}
    \item \textbf{Target marginal}: $M$ samples $\tilde{\theta}_N^{(1:M)} \sim p(\theta_N|\theta_T^{(0)})$ to estimate $p(h_K|\theta_T^{(0)},\pi)$
    \item \textbf{Full marginal}: $L$ contrastive samples $\theta^{(1)},\ldots,\theta^{(L)} \sim p(\theta)$ to estimate $p(h_K|\pi)$
\end{itemize}
Define
\begin{align}
\widehat p_M(h_K|\theta_T^{(0)},\pi)
&:= \frac{1}{M}\sum_{m=1}^{M} p(h_K|\theta_T^{(0)}, \tilde{\theta}_N^{(m)}, \pi), \\
\widehat p_L(h_K|\pi)
&:= \frac{1}{L+1}\sum_{\ell=0}^{L} p(h_K|\theta^{(\ell)}, \pi).
\end{align}

\begin{theorem}[Targeted sPCE Lower Bound]
\label{thm:targeted_spce}
Let $\theta^{(0)} = (\theta_T^{(0)}, \theta_N^{(0)}) \sim p(\theta)$ generate $h_K \sim p(h_K|\theta^{(0)}, \pi)$. Let $\theta^{(1)},\ldots,\theta^{(L)} \sim p(\theta)$ be independent contrastive samples, and let $\tilde{\theta}_N^{(1:M)} \sim p(\theta_N|\theta_T^{(0)})$ be nuisance samples. Define
\begin{equation}
\hat{\mathcal{L}}_K^{\text{tgt}}(\pi, L, M) = \E\left[\log \frac{\widehat p_M(h_K|\theta_T^{(0)},\pi)}{\widehat p_L(h_K|\pi)}\right],
\end{equation}
where the expectation is over $\theta^{(0)}$, the generated history $h_K$, the contrastive samples $\theta^{(1:L)}$, and the nuisance samples $\tilde{\theta}_N^{(1:M)}$.
Assume the relevant likelihoods are positive almost surely so that the logarithms are well defined, and that all displayed expectations are finite.
Then $\hat{\mathcal{L}}_K^{\text{tgt}}(\pi, L, M) \leq \mathcal{I}_K^{\text{tgt}}(\pi)$, with $\hat{\mathcal{L}}_K^{\text{tgt}} \to \mathcal{I}_K^{\text{tgt}}$ as $L, M \to \infty$.
\end{theorem}

\begin{proof}
The gap between the true objective and the practical bound is:
\begin{align}
\mathcal{I}_K^{\text{tgt}} - \hat{\mathcal{L}}_K^{\text{tgt}} &= \E\left[\log \frac{p(h_K|\theta_T^{(0)},\pi)}{p(h_K|\pi)} - \log \frac{\widehat p_M(h_K|\theta_T^{(0)},\pi)}{\widehat p_L(h_K|\pi)}\right] \\
&= \E\left[\log \frac{p(h_K|\theta_T^{(0)},\pi)}{\widehat p_M(h_K|\theta_T^{(0)},\pi)}\right] + \E\left[\log \frac{\widehat p_L(h_K|\pi)}{p(h_K|\pi)}\right].
\end{align}

\textbf{Term 1 (new nuisance-marginalization gap):} Condition on the realized $(\theta_T^{(0)}, h_K)$. Since $\tilde{\theta}_N^{(1:M)} \overset{\text{i.i.d.}}{\sim} p(\theta_N|\theta_T^{(0)})$,
\begin{equation}
\E\big[\widehat p_M(h_K|\theta_T^{(0)},\pi) \mid \theta_T^{(0)}, h_K\big]
= \int p(h_K|\theta_T^{(0)},\theta_N,\pi)\,p(\theta_N|\theta_T^{(0)})\,d\theta_N
= p(h_K|\theta_T^{(0)},\pi).
\end{equation}
Since $\log$ is concave, Jensen's inequality gives
\begin{equation}
\E\big[\log \widehat p_M(h_K|\theta_T^{(0)},\pi) \mid \theta_T^{(0)}, h_K\big]
\leq \log \E\big[\widehat p_M(h_K|\theta_T^{(0)},\pi) \mid \theta_T^{(0)}, h_K\big]
= \log p(h_K|\theta_T^{(0)},\pi),
\end{equation}
and hence
\begin{equation}
\E\left[\left.\log \frac{p(h_K|\theta_T^{(0)},\pi)}{\widehat p_M(h_K|\theta_T^{(0)},\pi)}\,\right|\,\theta_T^{(0)}, h_K\right] \geq 0.
\end{equation}
Taking expectations over $(\theta_T^{(0)}, h_K)$ yields the non-negativity of Term 1.

\textbf{Term 2 (standard sPCE gap):} This term is unchanged by the nuisance extension and is exactly the denominator gap analyzed by \citet{foster2021}. Reusing their change-of-measure and symmetry argument,
\begin{equation}
\E\left[\log \frac{\widehat p_L(h_K|\pi)}{p(h_K|\pi)}\right]
= \E_{p(h_K|\pi)}\left[\text{KL}\left(\tilde{p}(\theta^{(0:L)}|h_K) \,\|\, p(\theta^{(0:L)})\right)\right] \geq 0,
\end{equation}
where $\tilde{p}$ is the same mixture distribution as in the standard sPCE proof. Thus the second term is inherited directly from Foster et al.; the only new ingredient here is Term 1.

Since both terms are non-negative: $\hat{\mathcal{L}}_K^{\text{tgt}} \leq \mathcal{I}_K^{\text{tgt}}$.

\textbf{Convergence:} For the denominator term, we adopt the same bounded likelihood-ratio assumption used by \citet{foster2021}; as they note, it can be weakened, but we keep the simple uniform form here. For the new numerator term, we impose the analogous nuisance-ratio bound. Specifically, assume there exist constants $0<\kappa_{1,N} \le \kappa_{2,N}<\infty$ and $0<\kappa_1 \le \kappa_2<\infty$ such that
\begin{align}
\kappa_{1,N} \le \frac{p(h_K|\theta_T,\theta_N,\pi)}{p(h_K|\theta_T,\pi)} \le \kappa_{2,N} && \forall\,\theta_T,\theta_N,h_K, \\
\kappa_1 \le \frac{p(h_K|\theta,\pi)}{p(h_K|\pi)} \le \kappa_2 && \forall\,\theta,h_K,
\end{align}
Then both log-ratios are uniformly bounded:
\begin{align}
\left|\log \frac{p(h_K|\theta_T^{(0)},\pi)}{\widehat p_M(h_K|\theta_T^{(0)},\pi)}\right|
&\le \max\{ |\log \kappa_{1,N}|, |\log \kappa_{2,N}| \}, \\
\left|\log \frac{\widehat p_L(h_K|\pi)}{p(h_K|\pi)}\right|
&\le \max\{ |\log \kappa_1|, |\log \kappa_2| \}.
\end{align}
By the Strong Law of Large Numbers,
\begin{align}
\widehat p_M(h_K|\theta_T^{(0)},\pi) &\to p(h_K|\theta_T^{(0)},\pi) \quad \text{a.s. as } M\to\infty, \\
\widehat p_L(h_K|\pi) &\to p(h_K|\pi) \quad \text{a.s. as } L\to\infty.
\end{align}
The Bounded Convergence Theorem therefore gives Term 1 $\to 0$ and Term 2 $\to 0$, hence $\hat{\mathcal{L}}_K^{\text{tgt}} \to \mathcal{I}_K^{\text{tgt}}$.
\end{proof}

\section{Computational Budget Trade-offs}
\label{sec:budget_stub}

We briefly summarize how to allocate a fixed simulation budget across the outer minibatch size $B$, the number of contrastive samples $L$, and the number of nuisance samples $M$.
The goal is to approximate the gradient of the targeted expected information gain, $\nabla_\phi\,\mathcal{I}_K^{\text{tgt}}(\pi_\phi)$, as accurately as possible.
Let $\hat g_{B,L,M}(\phi) \in \mathbb{R}^{d_\phi}$ denote the resulting stochastic gradient estimator. For brevity, the bias, variance, and MSE rates below are stated for any fixed gradient coordinate.

\paragraph{Simulation budget (ODE rollouts).}
For mechanistic models, the dominant cost of evaluating $p(h_K|\theta,\pi)$ is simulating the system forward for $K$ design steps.
In our implementation (Algorithm~\ref{algo:DAD}), per episode we perform:
one rollout to generate $h_K$ under $\theta^{(0)}$,
$(L+1)$ rollouts for the denominator likelihoods under $\theta^{(0:L)}$, and
$M$ rollouts for the numerator likelihoods under $(\theta_T^{(0)}, \tilde{\theta}_N^{(m)})$ for $m=1,\ldots,M$.
The nuisance samples require separate ODE solves whenever nuisance parameters enter as initial conditions (e.g., initial biomass $C_{x,0}$) or appear in the dynamics (e.g., $\mu_{\max}$, $K_s$ in Haldane).
Thus, the compute budget is the total number of full-trajectory ODE rollouts per gradient update,
\begin{equation}
\label{eq:budget_rollouts}
\mathcal{C}_{\text{traj}} \;:=\; B\,(L+2+M),
\end{equation}
which is proportional to the number of ODE solver steps per update.

\paragraph{Asymptotic MSE of the gradient estimator.}
Let $g(\phi)=\nabla_\phi\,\mathcal{I}_K^{\text{tgt}}(\pi_\phi)$ denote the true targeted EIG gradient.
Let $\hat g_{B,L,M}(\phi)$ be the stochastic gradient used in training, obtained by differentiating the nested Monte Carlo objective in Theorem~\ref{thm:targeted_spce} and averaging over $B$ simulated episodes.
Under the regularity assumptions needed for the reparameterized gradient estimator and the second-order delta-method expansions used below---namely, interchange of differentiation and expectation, positivity of the Monte Carlo estimators almost surely, and sufficient moment/smoothness control (e.g., finite second and third moments for the relevant likelihood and gradient terms together with finite inverse moments of the Monte Carlo averages near their population means)---as $B,L,M\to\infty$ we have:

\emph{Bias.} At the objective level, Appendix~\ref{sec:nuisance} shows that the numerator term is a Jensen gap and the denominator term is the standard sPCE gap from Foster et al.~\citep{foster2021}. A second-order delta expansion of $\E[\log \widehat p_M\mid \theta_T^{(0)},h_K]$ around $p(h_K|\theta_T^{(0)},\pi)$ gives an $\mathcal{O}(M^{-1})$ numerator bias, while Foster et al. give the corresponding $\mathcal{O}((L+1)^{-1})$ denominator rate. Assuming these objective-level expansions may be differentiated with respect to $\phi$ under the expectation, the gradient bias satisfies:
\begin{equation}
\mathbb{E}[\hat g_{B,L,M}(\phi)] - g(\phi) = \mathcal{O}((L+1)^{-1}) + \mathcal{O}(M^{-1}).
\end{equation}

\emph{Variance.} The estimator averages $B$ independent episodes, so $\mathrm{Var}(\hat g) = \mathrm{Var}(\hat g^{(1)})/B$. The single-episode variance decomposes via the law of total variance into an outer term (from the ground truth and observation noise) and an inner term (from the contrastive and nuisance samples). Conditional on the outer randomness, the numerator contribution depends only on $\tilde{\theta}_N^{(1:M)}$ and the denominator contribution only on $\theta^{(1:L)}$, so their conditional covariance is zero. Moreover, $\nabla_\phi \log \widehat p_M$ is a smooth function of the pair $(\widehat p_M, \nabla_\phi \widehat p_M)$, and similarly for $\nabla_\phi \log \widehat p_L$. A multivariate delta method applied to these pairs gives conditional inner-variance contributions of orders $\mathcal{O}(M^{-1})$ and $\mathcal{O}((L+1)^{-1})$, respectively. Therefore:
\begin{equation}
\mathrm{Var}(\hat g_{B,L,M}(\phi)) = \mathcal{O}(B^{-1}) + \mathcal{O}\big((B(L+1))^{-1}\big) + \mathcal{O}((BM)^{-1}).
\end{equation}

\emph{MSE.} Since $\mathrm{MSE} = \mathrm{Bias}^2 + \mathrm{Var}$ and the cross term in $\mathrm{Bias}^2$ is bounded by the sum of the squared terms (by AM-GM):
\begin{equation}
\label{eq:grad_mse_scaling}
\mathrm{MSE}(\hat g_{B,L,M})
= \mathcal{O}(B^{-1})
+ \mathcal{O}\big((B(L+1))^{-1}\big)
+ \mathcal{O}((BM)^{-1})
+ \mathcal{O}((L+1)^{-2})
+ \mathcal{O}(M^{-2}).
\end{equation}

\paragraph{Optimal scaling under fixed ODE budget.}
Since all three variables $L$, $M$, $B$ consume ODE rollouts, we jointly optimize them under the budget constraint~\eqref{eq:budget_rollouts}.
Substituting $B = \mathcal{C}_{\text{traj}}/(L+2+M)$ into the leading terms of~\eqref{eq:grad_mse_scaling} gives the proxy
\begin{equation}
\mathrm{MSE}(L,M) \;\approx\; \frac{a\,(L+2+M)}{\mathcal{C}_{\text{traj}}} + \frac{c}{(L+1)^2} + \frac{d}{M^2},
\end{equation}
for problem-dependent constants $a,c,d>0$.
This reduction keeps only the leading terms. It is self-consistent: balancing the retained terms yields $L,M \propto \mathcal{C}_{\text{traj}}^{1/3}$ and $B \propto \mathcal{C}_{\text{traj}}^{2/3}$, under which the omitted inner-variance terms $\mathcal{O}((B(L+1))^{-1})$ and $\mathcal{O}((BM)^{-1})$ are both $\mathcal{O}(\mathcal{C}_{\text{traj}}^{-1})$, whereas the retained terms are $\mathcal{O}(\mathcal{C}_{\text{traj}}^{-2/3})$.
Setting $\partial\,\mathrm{MSE}/\partial L = 0$ and $\partial\,\mathrm{MSE}/\partial M = 0$ yields
\begin{equation}
L^\star + 1 = \left(\frac{2c\,\mathcal{C}_{\text{traj}}}{a}\right)^{1/3}, \qquad
M^\star = \left(\frac{2d\,\mathcal{C}_{\text{traj}}}{a}\right)^{1/3}, \qquad
B^\star = \frac{\mathcal{C}_{\text{traj}}}{L^\star + 2 + M^\star}.
\end{equation}
All three scale with $\mathcal{C}_{\text{traj}}$: $L^\star, M^\star \propto \mathcal{C}_{\text{traj}}^{1/3}$ and $B^\star \propto \mathcal{C}_{\text{traj}}^{2/3}$.
Asymptotically, the ratio $(L^\star+1)/M^\star = (c/d)^{1/3}$ depends on the relative magnitudes of the two Jensen gaps.
Equivalently, $B^\star \propto (L^\star+1)^2$ and $B^\star \propto (M^\star)^2$.
When the constants $c,d$ are unknown, a natural default is $c=d$, which gives $L^\star + 1 \approx M^\star$.

\paragraph{Gradient accumulation.}
In practice, GPU memory limits the number of concurrent ODE rollouts.
Gradient accumulation splits the $B$ episodes into $G$ micro-batches of size $B_{\text{micro}} = B/G$ and sums the resulting gradients, so at fixed $L$, $M$, and $B$ it leaves the estimator unchanged. Its role is simply to satisfy the memory constraint
\begin{equation}
(L + M) B_{\text{micro}} \lesssim \texttt{GPU\_MEM}.
\end{equation}
If memory fixes $B_{\text{micro}}$, then increasing $G$ increases the feasible total budget, $\mathcal{C}_{\text{traj}} = G B_{\text{micro}} (L+2+M)$. Re-optimizing $L$, $M$, and $B$ under this larger budget reduces the Jensen-gap bias, at the cost of roughly $G\times$ wall time per update.

\section{Policy Equivalence}
\label{sec:policy_equiv}

We show that the nested adaptive optimization criterion from Section~\ref{sec:criterion} is equivalent to optimizing over policies. Consider $K$ sequential input choices $u_1,\ldots,u_K$ with outcomes $Y_1,\ldots,Y_K$ and utility $U(u_{1:K},Y_{1:K})$. The nested adaptive objective is
\begin{equation}
V = \max_{u_1} \E_{Y_1|u_1}\Big[\max_{u_2} \E_{Y_2|Y_1,u_{1:2}}\Big[\cdots \max_{u_K} \E_{Y_K|Y_{1:K-1},u_{1:K}}[U]\cdots\Big]\Big].
\end{equation}
A deterministic non-anticipative policy $\pi = (\pi_1,\ldots,\pi_K)$ maps histories to inputs: $\pi_1 \in \mathcal{U}_1$ and $\pi_k: \mathcal{H}_{k-1} \to \mathcal{U}_k$ for $k \geq 2$, where $\mathcal{H}_{k-1}$ is the space of histories $(u_1,Y_1,\ldots,u_{k-1},Y_{k-1})$. The policy value is $J(\pi) = \E[U(\pi_1, \pi_2(h_1), \ldots, \pi_K(h_{K-1}), Y_{1:K})]$, with $\Pi$ the set of all such policies.

\begin{theorem}[Policy equivalence]
\label{thm:policy_equiv}
$V = \max_{\pi \in \Pi} J(\pi)$.
\end{theorem}

\begin{proof}
Define continuation values by backward induction:
\begin{align}
G_K(h_{K-1}) &= \max_{a_K \in \mathcal{U}_K} \E[U(u_{1:K-1}, a_K, Y_{1:K}) \mid h_{K-1}, a_K], \\
G_k(h_{k-1}) &= \max_{a_k \in \mathcal{U}_k} \E[G_{k+1}(h_{k-1}, a_k, Y_k) \mid h_{k-1}, a_k],
\end{align}
for $k = K{-}1, \ldots, 1$ (with $h_0 = \varnothing$). By construction, $G_1 = V$.

\emph{Part 1: $\max_\pi J(\pi) \leq V$.}
Fix any $\pi$ and let $V_k^\pi(h_{k-1})$ denote the expected utility under $\pi$ from stage $k$ onward. We show $V_k^\pi(h_{k-1}) \leq G_k(h_{k-1})$ for all $k$ by backward induction.

Base case ($k=K$): $\pi_K(h_{K-1})$ is one feasible action, so
\begin{equation}
V_K^\pi(h_{K-1}) = \E[U \mid h_{K-1}, \pi_K(h_{K-1})] \leq \max_{a_K \in \mathcal{U}_K} \E[U \mid h_{K-1}, a_K] = G_K(h_{K-1}).
\end{equation}
Inductive step: assume $V_{k+1}^\pi(h_k) \leq G_{k+1}(h_k)$ for all $h_k$. Then
\begin{align}
V_k^\pi(h_{k-1}) &= \E[V_{k+1}^\pi(h_k) \mid h_{k-1}, \pi_k(h_{k-1})] \\
&\leq \E[G_{k+1}(h_k) \mid h_{k-1}, \pi_k(h_{k-1})] \\
&\leq \max_{a_k \in \mathcal{U}_k} \E[G_{k+1}(h_k) \mid h_{k-1}, a_k] = G_k(h_{k-1}),
\end{align}
where the first inequality uses the induction hypothesis and the second uses the definition of $G_k$. Setting $k=1$ gives $J(\pi) = V_1^\pi \leq G_1 = V$.

\emph{Part 2: $V \leq \max_\pi J(\pi)$.}
Construct $\pi^\star$ by choosing argmax actions at each stage:
\begin{equation}
\pi_k^\star(h_{k-1}) = \argmax_{a_k \in \mathcal{U}_k} \E[G_{k+1}(h_{k-1},a_k,Y_k) \mid h_{k-1},a_k].
\end{equation}
We show $V_k^{\pi^\star}(h_{k-1}) = G_k(h_{k-1})$ for all $k$ by backward induction.

Base case ($k=K$): $\pi_K^\star$ achieves the argmax, so $V_K^{\pi^\star}(h_{K-1}) = G_K(h_{K-1})$.

Inductive step: assume $V_{k+1}^{\pi^\star}(h_k) = G_{k+1}(h_k)$ for all $h_k$. Then
\begin{align}
V_k^{\pi^\star}(h_{k-1}) &= \E[V_{k+1}^{\pi^\star}(h_k) \mid h_{k-1}, \pi_k^\star(h_{k-1})] \\
&= \E[G_{k+1}(h_k) \mid h_{k-1}, \pi_k^\star(h_{k-1})] \\
&= G_k(h_{k-1}),
\end{align}
where the first equality uses the induction hypothesis and the second uses the definition of $\pi_k^\star$. Setting $k=1$ gives $J(\pi^\star) = V_1^{\pi^\star} = G_1 = V$.
\end{proof}

\section{Bayesian Information Matrix Baseline}
\label{sec:bim}

\paragraph{Static BIM baseline.}
We use a Bayesian D-optimal design criterion based on the Fisher information matrix (FIM) as a baseline. Here $\theta$ denotes only the parameters that enter the dynamics and initial conditions---not the noise parameters. For a static input sequence $u_{1:K}$ and known observation noise covariance $\Sigma$, the FIM is
\begin{equation}
F(\theta, \Sigma, u_{1:K}) = J(\theta, u_{1:K})^\top \Sigma^{-1} J(\theta, u_{1:K}),
\end{equation}
where $J \in \mathbb{R}^{K \times q}$ is the sensitivity matrix with entries $J_{k,j} = \partial g(x(t_k)) / \partial \theta_j$, computed by differentiating the ODE solution. For the Monod bioreactor, $\Sigma = \sigma^2 I$ with $\sigma$ drawn from its prior.

When nuisance parameters are present in $\theta$, we project the information onto the target parameters via the Schur complement. Partitioning the posterior precision $P = F + \Lambda$ (where $\Lambda$ is the prior precision, with $\lambda_j = 12/(\theta_j^{\max} - \theta_j^{\min})^2$ for independent uniform priors) as
\begin{equation}
P = \begin{pmatrix} P_{TT} & P_{TN} \\ P_{NT} & P_{NN} \end{pmatrix},
\end{equation}
the marginal posterior precision for $\theta_T$ is
\begin{equation}
P_T = P_{TT} - P_{TN} P_{NN}^{-1} P_{NT}.
\end{equation}

The Bayesian D-optimal design maximizes the expected log-determinant:
\begin{equation}
u^\star_{\text{BIM}} = \argmax_{u_{1:K}} \; \E_{p(\theta, \Sigma)}\!\left[\log\det P_T(\theta, \Sigma, u_{1:K})\right],
\end{equation}
where the expectation averages over both the dynamics parameters and the noise covariance, approximated by prior samples. This is a standard criterion~\citep{chaloner1995} that assumes Gaussian posteriors (via the Laplace approximation) and does not account for sequential adaptation.
This formulation requires $\Sigma$ to be independent of $\theta$, which precludes our BIM implementation for the pharmacokinetic example, where the noise scales are nuisance parameters that interact with the dynamics parameters in the targeted sPCE objective.
We use the static BIM as a baseline for the Monod bioreactor (Table~\ref{tab:spce}) and the adaptive variant below for the DC motor (Table~\ref{tab:motor}).

\paragraph{Adaptive BIM baseline.}
The adaptive BIM strategy in Table~\ref{tab:motor} extends the static BIM criterion to a sequential, myopic setting. At step~$k$, given the history $(u_{1:k-1}, y_{1:k-1})$:
\begin{enumerate}
\item Compute a MAP estimate $\hat{\theta}_k = \argmax_\theta\, p(\theta \mid y_{1:k-1}, u_{1:k-1})$ via Adam optimization with ForwardDiff gradients (100 iterations, warm-started from $\hat{\theta}_{k-1}$; at $k=1$, initialized at prior midpoints).
\item Select $u_k = \argmax_{u} \log\det P_T(\hat{\theta}_k, [u_{1:k-1}, u])$ by grid search over 100 candidate voltages in $[0, 10]\,$V.
\item Execute $u_k$ on the true system and observe $y_k$.
\end{enumerate}
This replaces the prior expectation $\E_{p(\theta)}[\cdot]$ with a point evaluation at $\hat{\theta}_k$ (a Dirac delta approximation consistent with the Laplace assumption that BIM already makes). The MAP estimate is the natural choice: under the Laplace approximation, MAP coincides with the posterior mean. The strategy is myopic, optimizing one step ahead only.

The key computational difference is that adaptive BIM requires online posterior inference (MAP) and design re-optimization at \emph{each} experimental step, whereas the amortized sPCE policy requires only a single neural network forward pass. This makes adaptive BIM infeasible for real-time applications like the DC motor ($\Delta t = 10\,$ms), where the MAP + grid search takes milliseconds per step and exceeds the sampling interval by the later steps.

\section{Training and Evaluation Details}
\label{sec:training_details}

Training is implemented in Julia~\citep{bezanson2017julia} using Lux.jl~\citep{pal2023lux} for the neural network, Enzyme~\citep{moses2020enzyme} for automatic differentiation through the ODE solver, and Reactant~\citep{reactant2025} for GPU-accelerated compilation via MLIR/XLA. Training was performed on the VSC (Flemish Supercomputer Center) using NVIDIA H100 GPUs.

\paragraph{Architecture.}
The main examples share the same transformer architecture. Each design--observation pair $(u_k, y_k)$ is projected from $\mathbb{R}^2$ to $\mathbb{R}^{32}$ by a dense layer, then summed with 32-dimensional sinusoidal positional encodings. The transformer block consists of RMSNorm, causal multi-head attention (4~heads, dimension~32), and a feed-forward network ($32 \to 64$, GELU, $64 \to 32$). The output head is a dense layer $32 \to 1$ projecting the final hidden state to the design input, with a per-example initial bias (Table~\ref{tab:training_config}). Since the output is passed through a sigmoid, a negative bias initializes the untrained policy to produce small actions, which is more conservative for slow systems like bioreactors.

\paragraph{Optimizer.}
All examples use Adam with a cosine learning rate schedule: linear warmup from $10^{-5}$ to $3 \times 10^{-3}$ over 50~iterations, then cosine annealing back to $10^{-5}$.

\paragraph{Budget allocation.}
For each example, $L$, $M$, gradient accumulation $G$, and micro-batch size $B_{\text{micro}}$ were chosen using the heuristic in Appendix~\ref{sec:budget_stub}, subject to practical memory constraints. Gradient accumulation over $G$ micro-batches yields an effective batch size $B = G \times B_{\text{micro}}$. Table~\ref{tab:training_config} reports the settings used in the reported runs.

\paragraph{ODE integration.}
The dynamics are integrated with a fixed-step RK4 solver, using $N_{\text{sub}}$ substeps per design interval to ensure numerical stability.

A representative training loss curve is shown in Figure~\ref{fig:training_loss}.

\begin{table}[t]
\centering
\caption{Per-example training settings used in the reported runs. The main examples share the architecture and optimizer described above.}
\label{tab:training_config}
\small
\begin{tabular}{@{}lcccc@{}}
\hline
 & Monod & Haldane & Weibull & Motor \\
\hline
Iterations           & 1000 & 1000  & 1000  & 1000 \\
$\mathcal{C}_{\text{traj}}$ & 6.25M & 2.12M & 24M & 56M \\
$G$ (accum.\ steps)  & 8    & 16    & 8     & 8 \\
$L$ (contrastive)    & 464  & 161   & 725   & 960 \\
$M$ (nuisance)       & 466  & 162   & 727   & 962 \\
$B_{\text{micro}}$   & 6706 & 408   & 16506 & 29106 \\
$N_{\text{sub}}$ (RK4) & 50 & 50    & 10    & 10 \\
Output bias          & $-4.0$ & $-1.0$ & $-2.0$ & $0.0$ \\
\hline
\end{tabular}
\end{table}

\begin{figure}[t]
\centering
\includegraphics[width=0.7\columnwidth]{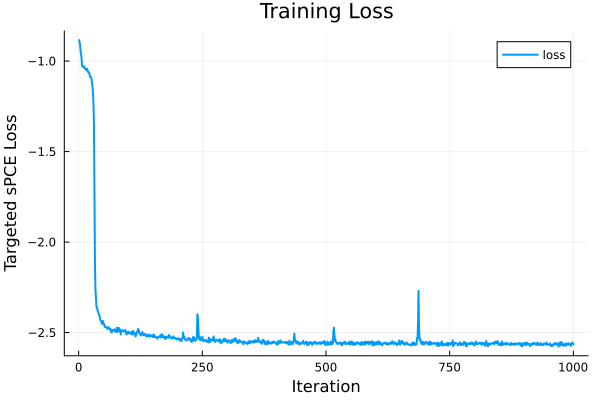}
\caption{Training loss curve (targeted sPCE objective) over iterations for the Monod bioreactor. The objective converges within approximately 1000 iterations.}
\label{fig:training_loss}
\end{figure}

\section{Architecture Ablation}
\label{sec:ablation}

We ablate the transformer-based policy on the Weibull PK model (Section~V) to isolate the contributions of attention and positional encoding. Three architectures are compared, all trained for 200 iterations with the same training setup:
\begin{itemize}
\item \textbf{Full Transformer}: the proposed architecture (attention + sinusoidal PE).
\item \textbf{Flat MLP}: a feedforward network that concatenates the full observation history into a single vector ($48 \to 64 \to 64 \to 1$; 7{,}361 parameters vs.\ 6{,}400 for the transformer).
\item \textbf{Transformer (no PE)}: the same transformer with positional encoding removed.
\end{itemize}

Evaluation uses 1{,}000 paired trials ($L{=}M{=}5{,}000$, $B{=}32$): all three models share the same sampled parameters and noise per trial, so differences are due solely to the architecture. Table~\ref{tab:ablation} reports the results.

\begin{table}[t]
\centering
\caption{Architecture ablation on the Weibull PK model (1{,}000 paired trials, $L{=}M{=}5{,}000$). Higher sPCE is better.}
\label{tab:ablation}
\small
\begin{tabular}{@{}lcc@{}}
\hline
Architecture & sPCE (mean $\pm$ SEM) & Paired $t$ vs.\ Full \\
\hline
Full Transformer      & $1.709 \pm 0.031$ & --- \\
Flat MLP              & $1.679 \pm 0.031$ & $t = 2.24$ ($p < 0.05$) \\
Transformer (no PE)   & $1.364 \pm 0.030$ & $t = 12.09$ ($p < 0.001$) \\
\hline
\end{tabular}
\end{table}

Positional encoding is the dominant factor: without it, the transformer cannot distinguish \emph{when} observations were collected, losing 0.35~nats of expected information. The attention mechanism provides a smaller but statistically significant benefit over the flat MLP (0.03~nats), suggesting that selectively weighting the observation history helps but is less critical than temporal awareness.

\fi

\end{document}